\documentclass[english]{article}
\usepackage[square,sort,comma,numbers]{natbib}
\usepackage{optidef}
\usepackage[T1]{fontenc}
\usepackage{geometry}
\usepackage{babel}
\usepackage{more}
\usepackage{comment}
\usepackage{subcaption}

\usepackage[utf8]{inputenc} 
\usepackage{url}            
\usepackage{booktabs}       
\usepackage{amsfonts}       
\usepackage{amsmath}
\usepackage{amsthm}
\usepackage{amssymb}
\usepackage{breqn}
\usepackage{graphicx}
\usepackage{mathtools}
\usepackage[ruled]{algorithm2e}
\usepackage{nicefrac}       
\usepackage{microtype}      
\usepackage[normalem]{ulem}
\usepackage{accents}

\usepackage{thmtools}
\usepackage{thm-restate}

\usepackage{natbib} 

\usepackage{cleveref}

\theoremstyle{plain}
\newtheorem{thm}{\protect\theoremname}
\theoremstyle{plain}

\theoremstyle{remark}

\theoremstyle{lemma}
\newtheorem{lem}[thm]{\protect\lemmaname}

\providecommand{\corollaryname}{Corollary}
\providecommand{\lemmaname}{Lemma}
\providecommand{\remarkname}{Remark}
\providecommand{\theoremname}{Theorem}
\providecommand{\lemmaname}{Lemma}

\def\R{\mathbb{R}}

 \author{Qi Deng\\
 Shanghai University of Finance and Economics\\
 \texttt{qideng@sufe.edu.cn}
 \and
 Yi Cheng \\
  Georgia Institute of Technology\\
 \texttt{cheng.yi@gatech.edu}
 \and
 Guanghui Lan\\
  Georgia Institute of Technology\\
 \texttt{george.lan@isye.gatech.edu}
 }
\begin{document}

\title{Optimal Adaptive and Accelerated Stochastic Gradient Descent}

\maketitle
\begin{abstract}
Stochastic gradient descent (\textsc{Sgd}) methods are
the most powerful optimization tools in training machine learning and deep
learning models. Moreover, acceleration (a.k.a. momentum)  methods and 
diagonal scaling (a.k.a. adaptive gradient) methods are the two main techniques to improve the 
slow convergence of \textsc{Sgd}. While empirical studies have demonstrated 
potential advantages of combining these two techniques, it remains unknown whether these methods can achieve the optimal rate of convergence for stochastic optimization.
In this paper, we present a new class of adaptive and accelerated stochastic gradient descent methods
and show that they exhibit the optimal sampling and iteration complexity for stochastic optimization.
More specifically, we show that diagonal scaling, initially designed to improve vanilla stochastic
gradient, can be incorporated into accelerated stochastic gradient descent to achieve the optimal
rate of convergence for smooth stochastic optimization. We also show that 
momentum, 
apart from being known to speed up the convergence rate of deterministic optimization, 
also provides us new ways of designing non-uniform and aggressive moving average schemes in stochastic optimization. 
Finally, we present some heuristics that help to implement adaptive accelerated stochastic
gradient descent methods and to further improve their practical 
performance for machine learning and deep learning. 
\end{abstract}

\section{Introduction}
In this paper we are interested in solving the following optimization
problem:
\begin{align}
\min_{x}&\quad f(x):=\mathbb{E}[F(x,\xi)] \nonumber \\
\text{s.t.}&\quad x \in \mathcal{X} \label{ilp}
\end{align}
where $\mathcal{X}$ is a closed convex set in $\R^d$, $F(\cdot,\xi):  \mathcal{X} \to \R$ is a continuously differentiable function
and $\xi \in \Xi \subset \R^p$ is a random variable.

The stochastic gradient descent method is obtained from gradient descent by replacing the
exact gradient $g(x)=\nabla f(x)$ with a stochastic gradient $G(x,\xi)=\partial_{x}F(x,\xi)$,
where $\xi$ is random samples. \textsc{Sgd} was initially presented in a seminal work by Robbins and Monro~\citep{RN218}
and later significantly improved in~\citep{RN177,RN337,RN202,RN176} through the incorporation of averaging and
adaption to problem geometry. 
\textsc{Sgd} methods  are now the
de-facto techniques to tackle the optimization problems for training large-scale machine learning and deep learning models. 

While momentum methods, first pioneered
by Polyak (\citep{RN316}) and later significantly improved in Nesterov's work \cite{RN184,RN185},
are well-known techniques to accelerate deterministic smooth convex optimization, only recently have these techniques
been incorporated into stochastic optimization starting from \cite{RN140,RN339}.
Under convex assumptions,
Nesterov's method and its stochastic counterparts have been shown to exhibit the optimal 
convergence rates for solving different classes of problems. 
These methods have also been recently generalized for solving nonconvex and stochastic optimization problems \cite{RN97}.
\cite{RN312} demonstrated the importance of momentum
and advantage of Nesterov's accelerated gradient method in training
deep neural nets. Recent works \cite{RN309,RN307} propose more robust accelerated \textsc{Sgd} with improved statistical error. 
In addition, many studies investigate convergence issue
for nonconvex optimization, including the theoretical
concern on the convergence to saddle point and how to escape from that.
\cite{RN322,RN333} show that gradient descent in general converges
to local optima. By utilizing second-order information, one can obtain
improved rate of convergence to approximate local minima. This includes
approaches based on Nesterov and Polyak's cubic regularization \cite{RN334,RN326,RN332},
or first-order method with accelerated gradient method as a sub-solver for escaping saddle
points \cite{RN323}.

In a different line of research, adaptive stepsizes to each decision variable
have been used to incorporate information about geometry of data more
efficiently. Built upon mirror-descent~\cite{RN176,RN177} and diagonal scaling,  
the earlier work~\cite{RN78} proposes \textsc{AdaGrad},
an adaptive subgradient method that has been widely applied to online
and stochastic optimization. Later \cite{RN313,RN308,RN314,RN321,RN320}
exploit adaptive \textsc{Sgd}, referred to as \textsc{Adam}, for optimizing deep neural networks. Their studies
suggest that exponential moving average, which leans towards the later
iterates, often outperform \textsc{AdaGrad}, which weighs the iterates
equally. The work~\cite{RN314} empirically combines accelerated gradient with \textsc{Adam}. 
However, there is no theoretical analysis and the interplay between momentum and
adaptive stepsize remains unclear.
In addition, the theoretical insight of exponential moving average
is not well understood. In particular, recent work~\cite{RN308} addresses
the non-convergent issue of \textsc{Adam}, a popular exponential averaging
\textsc{Sgd}, and proposes a modified averaging scheme with guaranteed convergence.
The work \cite{RN336} questions the generalization ability of adaptive
methods compared to \textsc{Sgd}. Later \cite{RN335} develops a strategy to
switch from \textsc{Adam} to \textsc{Sgd} for enhancing generalization performance while
retaining the fast convergence on the initial phase. 

Our goal in this paper is to develop a new accelerated stochastic
gradient descent, namely \textsc{A2Grad}, by showing in theory how adaptive stepsizes should
be chosen in the framework of accelerated methods. Our main contribution are as follows:
\begin{enumerate}
  \item A clean complexity view of adaptive momentum methods. We show that adaptive stepsizes can improve the convergence of stochastic component while Nesterov's accelerated method improves the deterministic and smooth part. In particular, \textsc{A2Grad} not only achieves the optimal $\mathcal{O}(\frac{L}{k^{2}}+\frac{\sigma}{\sqrt{k}})$
worst-case rate of convergence  for stochastic optimization
as shown in \cite{RN140}, but can also improve significantly the second term by adapting to 
data geometry. Here $L$ denotes the Lipschitz constant for the gradient $\nabla f$, 
$\sigma$ denotes the variance of stochastic gradient and $k$ denotes the iteration number. 
  \item Our paper provides a more general framework to analyze adaptive accelerated methods---not only including \textsc{AdaGrad}-type moving average, but also providing new ones---with unified theory. 
We show that momentum also guide us to design incremental and adaptive stepsizes that beyond the choices in existing work. One can choose between uniform average and nonuniform average up to quadratic weights in the adaptive stepsize while retaining the optimal rate.
\item  Our theory also includes exponential moving average as a special case and for the first time, has obtained the nearly-optimal convergence under sub-Gaussian assumption. The rate can be further improved to optimal under assumption of bounded gradient norm.
\end{enumerate}

The rest of the paper proceeds as follows. Section \ref{sec:sgd}
introduces the accelerated \textsc{Sgd} and adaptive \textsc{Sgd}. Section \ref{sec:Adaptive-ASGD}
states our main contribution which presents a new \textsc{A2Grad} framework of developing
accelerated \textsc{Sgd}. Section \ref{sec:diag} discusses several adaptive gradient strategies. Section \ref{sec:Experiments} conducts experiments
to the empirical advantage of our approach. Finally we draw several conclusions in Section \ref{sec:Conclusion}.

\paragraph{Notation.} For a vector $x\in \R^d$, we use $x^2$ to denote element-wise square and $\sqrt{x}$ to denote the element-wise square root. For vectors $x$ and  $y$, $xy$ and $\frac{x}{y}$ denote the element-wise multiplication and division respectively. 
With slight abuse of notation, $x_i$ denotes the $i$-th coordinate of vector $x$, and $x_{k, i} $ denotes the $i$-th coordinate of vector $x_k$. For a sequence of vectors $x_k, x_{k+1},...,x_t$ where $k\le t$, we use $x_{k:t,i}$ to denote the vector $[x_{k,i},x_{k+1,i},...,x_{t,i}]^T$ .

\section{Stochastic gradient descent\label{sec:sgd}}
The last few years has witnessed the great success of \textsc{Sgd} type methods
in the field
of optimization and machine learning. For many learning tasks,
\textsc{Sgd} converges slowly and momentum method improves
\textsc{Sgd} by adding inertia of the iterates to accelerate the optimization
convergence. However, classical momentum method does not guarantee
optimal convergence, except for some special cases. Under the convex
assumption, Nesterov developed the celebrated accelerated gradient (NAG) method
 which achieves the optimal rate of convergence.
By presenting a unified analysis for smooth and stochastic optimization, and 
carefully designing the stepsize policy, \cite{RN140} studies a variant
of Nesterov's accelerated gradient method that exhibits the optimal convergence
for convex stochastic optimization. The optimal accelerated stochastic approximation method
in  \cite{RN140} converges at
the rate of $\mathcal{O}(\frac{L}{k^{2}}+\frac{\sigma}{\sqrt{k}})$,
much better than vanilla \textsc{Sgd} with a rate of $\mathcal{O}(\frac{L}{k}+\frac{\sigma}{\sqrt{k}})$
for solving ill-conditioned problems with large constant $L$. 
Later work \cite{RN312} investigates the performance of \textsc{Sgd} on training deep neural networks, and demonstrates the
importance of momentum, in particular, the advantage of Nesterov's acceleration in
training deep neural networks.

%


\begin{algorithm}[h]
\KwIn{$x_0$, $v_0$}
\For{k=0,1,2,...,K}{
Obtain sample $\xi_k$ and compute $G_k\in \nabla F(x_k,\xi_k)$\;
$v_{k} = \psi(G^2_0, G^2_1, G^2_2, ..., G^2_k)$\;
$\tilde{G}_{k} = \phi(G_0, G_1, G_2, ...,G_k)$\;
$x_{k+1}  =x_{k} - \beta_{k} \tilde{G}_k / \sqrt{v_k}$\;
}
\KwOut{$x_{K+1}$}
\caption{Adaptive method\label{alg:adaptive}}
\end{algorithm}

However, since both \textsc{Sgd} and momentum methods only depend
on a single set of parameters for tuning all the coordinates, they can be 
inefficient for ill-posed and high dimensional problems. 
To overcome these
shortcomings, adaptive methods perform diagonal scaling on
the gradient by incorporating the second moment information. 
We describe a general framework of adaptive gradient in Algorithm \ref{alg:adaptive}, where $\psi$ and $\phi$ are some weighing functions. 
For example, the original \textsc{ AdaGrad} (\cite{RN78}) takes adaptive stepsize in the following form:
\begin{align*}
v_{k} & =v_{k-1}+G_{k}^{2}\\
x_{k+1} & =x_{k}-\frac{\beta_{k}}{\sqrt{v_{k}}}G_k.
\end{align*} 
In deep learning, the optimization landscape can be very different
from that of convex learning models. It seems to be intuitive to put
higher weights on the later iterates in the optimization trajectory. 
Consequently,  adaptive methods including \textsc{Rmsprop}\cite{RN320},
\textsc{Adadelta} \cite{RN313}, \textsc{Adam} \cite{RN321} and several others propose more practical adaptive methods with slightly different schemes to exponentially average the past second moment.

Both momentum and adaptive gradient improve the performance of \textsc{Sgd} and quickly become popular in machine learning. 
However, these two techniques seem to improve \textsc{Sgd} on some orthogonal directions. 
Momentum
aims at changing the global dynamics of learning system while adaptive gradient
aims at making \textsc{Sgd} more adaptive to the local geometry. 
The convergence of adaptive gradient is first established for  \textsc{AdaGrad} and later extended to adaptive momentum methods with exponential moving average.
However the proof of all these work seems to be based on regret analysis and subgradient methods. When applying momentum based adaptive gradients for smooth optimization, the theory only guarantees a convergence rate of $O(\frac{1}{\sqrt{k}})$, much worse than the 
$O(\frac{1}{k^2})$ achieved by other accelerated methods.
Thus it is natural to ask the question:  \emph{Can momentum truly help improve the convergence of adaptive gradient methods?}

\section{Adaptive ASGD\label{sec:Adaptive-ASGD}}
To answer this question, we will base our study on accelerated gradient method using Nesterov's momentum.
Our main goal in this section is to develop a novel accelerated stochastic gradient
method by combining the advantage of Nesterov's accelerated gradient
method and adaptive gradient method with stepsize tuned per parameter.
We will discuss the interplay of these two strategies and develop new moving average
schemes to further accelerate convergence.

We begin with several notations. Define
$D_{\phi}(y, x)=\phi(x)-\phi(y)-\left\langle \nabla\phi(y),x-y\right\rangle $
as a Bregman divergence where the  convex smooth function $\phi(\cdot)$ is called a proximal function or distance generating function in the literature of mirror descent. We assume that $\phi(x)$ is 1-strongly convex with respect to the semi-norm $\|\cdot \|_{\phi}$, or equivalently, $D_\phi(x,y) \ge \frac{1}{2} \|x-y\|_\phi^2$. 
Also let $\|\cdot \|_{\phi*}$ be the dual norm associated with $\|\cdot \|_{\phi}$.  
Many popular optimization algorithms, such as gradient descent, adopt $D_\phi(x,y)=\frac{1}{2}\|x-y\|^2$, simply by setting proximal function according to the Euclidean norm: $\phi(x)=\frac{1}{2}\|x\|^2$.
In contrast, adaptive gradient methods apply data-driven $\phi_k(x)$ $(k\ge0)$ which is adaptive to the past (sub)gradient.

 In Algorithm \ref{alg:a2grad}, we propose
a general algorithmic framework: adaptive accelerated stochastic gradient (\textsc{A2Grad})
method, of which the main idea is to apply extra adaptive proximal function in the process of accelerated gradient methods while retaining the optimal convergence imposed by Nesterov's momentum.
Examining step (\ref{eq:a2grad-2}) in the algorithm, the major
difference between \textsc{A2Grad} and the other stochastic accelerated gradient descent 
exists in that the new update couples the two Bregman distance $D(x_{k},x)$
and $D_{\phi_k}(x_{k},x)$, controlled by stepsize $\gamma_{k}$ and
$\beta_{k}$ respectively. The first one $D(y,x)$ is short for $D_\psi(y, x)$, where $\psi$ can be any fixed convex proximal function during the optimization iterations. 
For practical purpose, in this paper we will use $\psi(x)=\frac{1}{2}\|x\|^2$ and then skip the subscript for the sake of simplicity.
On the other hand, $D_{\phi_k}(x_{k},x)$ uses a variable proximal function, allowing gradual adjustment when algorithm proceeds. 
As a consequence, the mirror descent update is a joint work of both Nesterov's momentum and adaptive gradient.
Although the view of two Bregman divergence seems straightforward, it provide us with a clearer vision on the interaction between Nesterov's term and adaptive gradient. 
As we shall see later, our framework leads to novel  adaptive moving average with sharper theoretical guarantee.

In the following theorem, we provide the first main result on the general convergence property of \textsc{A2Grad}.  The proof of the theorem is left in Appendix for brevity.
\begin{algorithm}[h]
\KwIn{ $x_{0}$, $\bar{x}_{0}$, $\gamma_k ,\beta_k >0$} 
\For{$k=0,1,2\ldots K$}{
Update:
\begin{equation}
  \underline{x}_{k} =(1-\alpha_{k})\bar{x}_{k}+\alpha_{k}x_{k}\label{eq:a2grad-1}
\end{equation}
Sample $\xi_{k}$, compute $\underline{G}_{k}\in\nabla F(\underline{x}_{k},\xi_{k})$
and $\phi_{k}(\cdot)$, then update:

\begin{align}
x_{k+1} & =\argmin_{x \in \mathcal{X}}\left\{ \left\langle \underline{G}_{k},x\right\rangle +\gamma_{k}D(x_{k},x)+\beta_{k}D_{\phi_{k}}(x_{k},x)\right\} \label{eq:a2grad-2}\\
\bar{x}_{k+1} & =(1-\alpha_{k})\bar{x}_{k}+\alpha_{k}x_{k+1}\label{eq:a2grad-3}
\end{align}

}
\KwOut{$\bar{x}_{K+1}$}\caption{Adaptive accelerated stochastic gradient (\textsc{A2Grad}) algorithm\label{alg:a2grad}}
\end{algorithm}

\begin{restatable}{thm_re}{main}
\label{thm:main-thm} 
In Algorithm \ref{alg:a2grad}, if $f$ is convex and Lipschitz smooth with constant $L$, and $\{\alpha_{k}\}$ and
$\{\gamma_{k}\}$ satisfy the following conditions: 
\begin{align}
L\alpha_{k} & \le\gamma_{k},\label{eq:alpha_gamma_bound}\\
\lambda_{k+1}\alpha_{k+1}\gamma_{k+1} & \le\lambda_{k}\alpha_{k}\gamma_{k},\label{eq:alpha_gamma_mono}
\end{align}
where the sequence $\{\lambda_k\}$ is defined by 
$\lambda_{k}=1/\prod_{i=1}^{k}(1-\alpha_{i}),k=1,2,3...,\; \text{and\ }\lambda_{0}=1, $
then 
\begin{dmath}
\lambda_{K}\left[f(\bar{x}_{K+1})-f(x)\right] \le(1-\alpha_{0})\left[f(\bar{x}_{0})-f(x)\right]
+\alpha_{0}\gamma_{0}D(x_{0},x)+\sum_{k=0}^{K}\lambda_{k}[\frac{\alpha_{k}\|\delta_{k}\|_{\phi_{k}*}^{2}}{2\beta_{k}} 
+\alpha_{k}\left\langle \delta_{k},x-x_{k}\right\rangle +\alpha_{k}R_{k}],
\label{main-recur}
\end{dmath}
\normalsize
where $\delta_{k}=\underline{G}_k -\nabla f(x_k)$ and $R_{k}=\beta_{k}D_{\phi_k}(x_{k},x)-\beta_{k}D_{\phi_k}(x_{k+1},x).$ 
\end{restatable}

Let us make a few comments about Theorem \ref{thm:main-thm}.
First, the parameters $\{\gamma_k\}$ and $\{\alpha_k\}$ satisfy some recursion that are typical for Nesterov's accelerated method \cite{RN249}; their choices  are crucial to guarantee optimal convergence in convex smooth optimization. 
Second,  the dynamics of adaptive gradient exhibits a new pattern, crucially differing from many other adaptive methods such as \textsc{AdaGrad, Adam} and \textsc{Nadam}.  
Thanks to the momentum, the pattern renders a novel interplay of  proximal function $\phi_k$  with  variance of stochastic gradient: $\|\delta_k\|_{\phi_k^*}$, instead of with (sub)gradient norm $\|\underbar{G}_k\|_{\phi_k^*}$. 
Finally, we need to mention that the discussion so far is still conceptual, since the parameters have not been fixed yet. In particular, one important remaining issue is to select appropriate proximal function $\phi_k$. Hereafter, the paper will be devoted to deriving more specific
 choices of these parameters.

\section{Diagonal scaling and moving average schemes\label{sec:diag}}

There are many different ways to choose the underlying proximal function $\phi_k$. For practical use, we adopt diagonal scaling function :$\phi_k(x)=\frac{1}{2}\|x\|^2_{h_k}=\frac{1}{2} \langle x^T, \text{Diag}(h_k) x \rangle$, where $h_k\in\mathbb{R}^d$ and $h_{k,i}>0$, $1\le i \le d$. Consequently, the Bregman divergence has the following form:
$ D_{\phi_{k}}(x, y)=\frac{1}{2}\langle (x-y)^{T}\text{Diag}(h_{k})(x-y)\rangle $ and $\ D(x,y)=\frac{1}{2}\|x-y\|^2.$
\begin{algorithm}[h]
Compute $h_k$ and replace the update of $x_{k+1}$ in (\ref{eq:a2grad-2}) by
\begin{equation*}
x_{k+1} = \Pi_{\mathcal{X}} (x_k - \frac{1}{\gamma_k +  \beta_k  h_k} \underline{G}_{k}),
\end{equation*}
where $\Pi_{\mathcal{X}}$ denotes the projection operator over $\mathcal{X}$.
\caption{\textsc{A2Grad} with diagonal scaling \label{alg:a2grad_diag}}
\end{algorithm}

In addition, a crucial step to guarantee convergence is to ensure
that the sum $\sum_{k}\lambda_{k}\alpha_{k}R_{k}$ does not grow too
fast. One sufficient condition is to ensure that for any $k\ge0$:
\begin{equation}
\lambda_{k}\alpha_{k}\beta_{k}D_{\phi_{k}}(x_{k},x)\ge\lambda_{k-1}\alpha_{k-1}\beta_{k-1}D_{\phi_k}(x_{k-1},x). \label{monotone}
\end{equation}

We describe a specific variant of \textsc{A2Grad} using diagonal scaling function in Algorithm \ref{alg:a2grad_diag}, and provide a general scheme to choose its parameters. 
The convergence of Algorithm \ref{alg:a2grad_diag} is summarized
in the following theorem.

\begin{restatable}{thm_re}{diagcvg}
\label{thm:bound-diag} Let $x^*$ be an optimal solution. Also assume $ \|x_k-x^*\|_\infty^2 \le B$ for some $B >0$ (this assumption 
is satisfied, e.g., when $\mathcal{X}$ is compact). In Algorithm \ref{alg:a2grad_diag}, if $\alpha_{k}=\frac{2}{k+2}$,  $\gamma_{k}=\frac{2L}{k+1}$
and $\{\beta_k, h_k\} $ satisfy the monotone property $$\lambda_{k+1}\alpha_{k+1} \beta_{k+1} h_{k+1,i}\ge \lambda_k\alpha_k \beta_k h_{k,i},$$ 
$$\forall k=0,1,2,..., 1\le i \le d, $$ 
then we have 
\small
\begin{dmath*}
\mathbb{E}\left[f(\bar{x}_{K+1})-f(x^*)\right] 
 \le \frac{2L\|x^*-x_0\|^2}{(K+1)(K+2)}
 +\frac{2B\beta_{K}\mathbb{E}\left[\sum_{i=1}^{d}h_{K}^{(i)}\right]}{K+2}
 +\sum_{k=0}^{K}\mathbb{E}\left\lbrace\frac{(k+1)\|\delta_{k}\|_{h_{k}*}^{2}}{\beta_k(K+1)(K+2)}\right\rbrace.
\end{dmath*}
\end{restatable}

The main message of Theorem \ref{thm:bound-diag} is that the complexity rate of \textsc{A2Grad} can be viewed as the sum of deterministic dynamics and stochastic dynamics. 
The proximal function  $h_k$ exhibits an interplay with variance $\delta_k$ instead of  $\|\underbar{G}_k \|^2$, which is commonly seen in existing adaptive methods. 
Notably, such interplay between $h_k$ and $\delta_k$ is essential to make efficient use of momentum, because now momentum term can significantly accelerate convergence of deterministic dynamics while adaptive gradient will concentrate on the stochastic dynamics.
Moreover, thanks to the momentum, the bound on stochastic part is more complicated than that of adaptive gradient.  As we will see soon, it provides us a rich source of adaptive stepsizes.


\subsection*{Uniform vs nonuniform moving average}

 Our goal is  to design  specific adaptive stepsizes with guaranteed theoretical convergence.  
 For the moment let us assume that the full gradient $g_{t}$ is known.  
 The adaptive function $h_k$ will be chosen as a square root of the weighted sum of  sequence $\{\delta_k^2\}$.
Immediately we conclude the convergence rate in the following
corollary. 

\begin{restatable}{cor_re}{polyavg}
\label{cor:poly-avg}Under the assumption of Theorem \ref{thm:bound-diag}, let $\beta_{k}=\beta$ and $$h_{k}=\frac{1}{(k+1)^{q/2}}\sqrt{\sum_{\tau=0}^{k}(\tau+1)^q\delta_{\tau}^{2}}\ \text{ with } \ q\in[0, 2],$$
then we have
\begin{dmath*}
\mathbb{E}\left[f(\bar{x}_{K+1})-f(x^*)\right] 
 \le\frac{2L\|x-x_{0}\|^{2}}{(K+1)(K+2)} +\frac{2\mathbb{E}\left[\sum_{i=1}^{d}\|\delta_{0:K,i}\|\right]}{\beta(K+2)}
+\frac{2\beta B\ \mathbb{E}\left[\sum_{i=1}^{d}\|\delta_{0:K,i}\|\right]}{K+2}.
\end{dmath*}
\end{restatable}

In Corollary \ref{cor:poly-avg}, the stochastic part is related to the expected sum of $\|\delta_{0:K,i}\|$, for $i=1,2,...,d$. Under the mild assumption of bounded variance, namely, for $x\in\mathcal{X}$,  $\mathbb{E}_\xi\|
G(x,\xi)-g(x,\xi)\|^2\le \sigma^2$, where $\sigma>0$, we have 
$$\mathbb{E}[\|\delta_{0:K,i}\|]\le \sqrt{\sum_{k=0}^K\mathbb{E}[\delta_{k,i}^2]}\le \sqrt{K+1}\sigma 
$$ by Jensen's inequality.
Hence the stochastic part obtains the $\mathcal{O}(\frac{1}{\sqrt{K}})$ rate of convergence.  Moreover, the convergence rate also suggests that there exists an optimal $\beta$ to minimize the bound. However, such optimal value is rarely known and will be a hyper-parameter for tuning in the practice.

Furthermore, Corollary \ref{cor:poly-avg} allows us to choose a family of adaptive functions. Next we will consider two specific choices.
By setting $q=0$, we arrive at uniform moving average for forming the diagonal function: $h_{k}=\sqrt{\sum_{\tau=0}^{k}\delta_{\tau}^{2}}.$
We describe \textsc{A2Grad} with this choice in 
Algorithm \ref{alg:a2grad-uni}. 


\begin{algorithm}[h]
\KwIn{$v_{-1}=0$ and the rest of other parameters}
The $k$-th step of Algorithm \ref{alg:a2grad_diag}:
\begin{align*}
v_{k} & =v_{k-1}+\delta_{k}^{2}\\
h_{k} & =\sqrt{v_{k}}
\end{align*}

\caption{\textsc{A2Grad-uni}: Adaptive ASGD with uniform moving average\label{alg:a2grad-uni}}
\end{algorithm}

Setting $q=2$, we arrive at an incremental moving average with quadratic
weight: $h_k$ by $h_{k}=\frac{1}{(k+1)}\sqrt{\sum_{\tau=0}^{k}(\tau+1)^{2}\delta_{\tau}^{2}}. $ We describe \textsc{A2Grad} with such incremental scheme in Algorithm  \ref{alg:a2grad-inc}.
Compared with uniform averaging, the scaling vectors $\delta_k$ will first shrink by a factor $\frac{k^2}{(k+1)^2}$ before adding the new iterate. 
On the long run, the earlier scales will have a relative weights decaying at the rate of $\mathcal{O}(\frac {1}{k^2}$). To our best knowledge, this is the first study on using nonlinear polynomial weights in the work of adaptive gradient methods, and the first study of using nonlinear weights without introducing an additional $\mathcal{O}(\log{k})$ factor in the complexity bound. 


\begin{algorithm}[h]
\KwIn{$v_{-1}=0$ and the rest of other parameters}
The $k$-th step of Algorithm \ref{alg:a2grad_diag}:
\begin{align*}
v_{k} & = k^2/(k+1)^2 v_{k-1}+\delta_{k}^{2}\\
h_{k} & =\sqrt{v_{k}}
\end{align*}
\caption{\textsc{A2Grad-inc}: Adaptive ASGD with incremental moving average (quadratic weight) \label{alg:a2grad-inc}}
\end{algorithm}

\subsection*{Exponential moving average}
Besides the uniform and nonuniform schemes discussed above, 
it is important to know whether exponential moving average---a popular choice in deep learning area---can be combined with Nesterov's accelerated method with justified theoretical convergence.

A crucial difference between the family of adaptive gradients by exponential moving averages and the earlier \textsc{AdaGrad} in how the diagonal terms are accumulated. Typically, exponential average methods apply the update 
\begin{equation}
v_k=\theta v_{k-1} + (1-\theta) G_k^2,\quad h_k=\sqrt{v_k} \label{exp-ada}  
\end{equation}
with $\theta \in (0, 1)$ for deriving the diagonal term. Although exp-avg methods have received popularity in deep learning field, their theoretical understanding is limited and their convergence performance remains a question. In particular, recent \cite{RN308} provides a simple convex problem on which \textsc{Adam} fails to converge easily. 
The intuition of such phenomenon is as follows. Generally, to ensure algorithm convergence, SGD reduces the influence of variance of stochastic gradients by applying diminishing stepsize $\eta_k: \eta_k \rightarrow 0$  as $k\rightarrow 0$.  In \textsc{Adam} and many other exp-avg methods, the effective stepsize $\eta_k$ is controlled by the scaling term $h_k$ with $\eta_k\propto 1/h_k$. However, using exponential average (\ref{exp-ada}) we are unable to guarantee that stepsize will be diminishing stepsize.

To equip adaptive methods with exponential average, it is important to resolve the non-convergence issue observed in \textsc{Adam}. 
In view of Theorem \ref{thm:bound-diag}, it suffices to choose $h_k$ such that the monotone relation (\ref{monotone}) is satisfied.  
Towards this goal, we present in Algorithm \ref{alg:a2grad-exp}, a variant of \textsc{A2Grad} with exponential moving average. Overall, Algorithm  \ref{alg:a2grad-exp} is  different from the earlier algorithms due to an extra step.
 By introducing  auxiliary $\tilde{v}_k$, the scaling term $v_k$ is monotonically growing. This together with the explicit weight $\sqrt{k+1}$ in expressing $h_k$ guarantees the relation (\ref{monotone}).

\begin{algorithm}[h]
\KwIn{$\tilde{v}_{-1}=0$ and the rest of other parameters}
The $k$-th step of Algorithm \ref{alg:a2grad_diag}:
\begin{align*}
\tilde{v}_{k} &=\begin{cases}
\delta_{k}^{2} & \text{if\ }k=0\\
\rho\tilde{v}_{k-1}+(1-\rho)\delta_{k}^{2} & \text{otherwise}
\end{cases}\\
v_{k} &=\max\left\{ \tilde{v}_{k},v_{k-1}\right\} \\
h_{k} &=\sqrt{(k+1)v_{k}}
\end{align*}
\caption{\textsc{A2Grad-exp}: Adaptive ASGD with exponential moving average\label{alg:a2grad-exp}}
\end{algorithm}

We describe the convergence rate of Algorithm \ref{alg:a2grad-exp} in the following corollary, and leave the proof details in Appendix.
\begin{restatable}{cor_re}{expavg}
\label{cor:exp-avg}Under the assumption of Theorem \ref{thm:bound-diag}, if $\beta_{k}=\beta$, $\rho\in(0,1)$ and $h_{k}$ in Algorithm \ref{alg:a2grad-exp},
then we have
\begin{dmath}
\mathbb{E}\left[f(\bar{x}_{K+1})-f(x^*)\right]  
\le
\frac{2L\|x^*-x_{0}\|^{2}}{(K+1)(K+2)}
+\frac{2\beta B\sum_{i=1}^{d}\mathbb{E}\left[\max_{0\le k\le K}\left|\delta_{k,i}\right|\right]}{\sqrt{K+2}}
+\frac{\sum_{i=1}^{d}\mathbb{E}\left[\|\delta_{0:K,i}\|_1\right]}{2\beta(1-\rho)\sqrt{K+1}(K+2)}. \label{eq:converge-exp-avg}
\end{dmath}
Moreover, if we assume that the distribution of error $\delta_i$ is sub-Gaussian with variance factor $\bar{\sigma}$, namely, the moment generating function  satisfies:
$ \mathbb{E}e^{t \delta_i}\le e^{t^2 \bar{\sigma}^2/2} , \forall t>0$, then we have
\begin{dmath}
\mathbb{E}\left[f(\bar{x}_{K+1})-f(x^*)\right] 
\le \frac{2L\|x^*-x_0\|^2}{(K+1)(K+2)}
+ 2\beta B\frac{\sqrt{2\log(2(K+1))}\bar{\sigma}}{\sqrt{K+2}}
+\frac{\sqrt{2\pi}d \bar{\sigma}}{2\beta(1-\rho)\sqrt{K+2}}.
\end{dmath}
If we assume bounded gradient, namely, $\exists C>0$, s.t. $\|G(x, \xi)\|_\infty \le C$, then 
 \begin{dmath}
\mathbb{E}\left[f(\bar{x}_{K+1})-f(x^*)\right] 
\le \frac{2L\|x^*-x_0\|^2}{(K+1)(K+2)}
+ \frac{4\beta BCd}{\sqrt{K+2}}
+\frac{Cd}{\beta(1-\rho)\sqrt{K+2}}.
\end{dmath}
\end{restatable}

\paragraph{Remark} The above Corollary shows that our algorithm attains an $O(\sqrt{\log(K)/K})$ rate of complexity with a relatively weak sub-Gaussian assumption. 
Alternatively, many adaptive gradient methods adopts the bounded gradient norm assumption. For example, \cite{RN308} has shown the convergence  of \textsc{AMSGrad} using regret analysis.  
Their result can be translated into an $O(\sqrt{\log(K)/K})$ rate using online to batch conversion. In comparison, the last part of Corollary \ref{cor:exp-avg} proves  a much better $O(1/\sqrt{K})$ rate of complexity under the same assumption. 
To our best knowledge, this seems to be the first adaptive gradient with exponential moving average that obtains $O(1/\sqrt{K})$ rate.


\section{Experiments\label{sec:Experiments}}
Our goal of this section is  to demonstrate the advantage of our proposed algorithms when compared with existing adaptive gradient methods.
 Specifically, we use logistic regression and neural network to demonstrate the performance of our proposed algorithms for convex and nonconvex problems respectively. 
 We will investigate the performance of three variants of our algorithms, namely, \textsc{A2Grad-uni}, \textsc{A2Grad-inc} and \textsc{A2Grad-exp}
For comparison, we choose  \textsc{Adam} and \textsc{AMSGrad} which employ exponential moving average. We skip \textsc{AdaGrad} since it always gives inferior performance.
\paragraph{Practical implementation.} Following the same manner in deep learning platforms (like PyTorch), namely, the parameters taking gradient are the same to evaluate objectives, we investigate performance of $\underline{x}_{k}$. Observing that
 \begin{eqnarray*}
\underline{x}_{k+1} &=&(1-\alpha_{k+1})\bar{x}_{k+1}+\alpha_{k+1}x_{k+1} \\
  &=&(1-\alpha_{k+1})\left[\underline{x}_{k}+\alpha_{k}(x_{k+1}-x_{k})\right]+\alpha_{k+1}x_{k+1}\\
  &=&(1-\alpha_{k+1})\underline{x}_{k}+\alpha_{k+1}x_{k+1}\\
  & &-(1-\alpha_{k+1})\frac{\alpha_{k}}{\gamma_{k}+\beta_{k}h_{k}}\underline{G}_{k}  
\end{eqnarray*}
 Hence we can reduce to two variables by eliminating $\bar{x}_{k+1}$.
Let us denote $y_k=\underline{x}_k$ and $G_k=\underline{G}_k$ for simplicity, we present the rewritten implementation in Algorithm \ref{alg:a2grad-imp}.

\begin{algorithm}[h]
\KwIn{the rest of other parameters}
The $k$-th step of Algorithm \ref{alg:a2grad_diag}:
\begin{align*}
x_{k+1} &=x_k -\frac{1}{\gamma_k +  \beta_k  h_k}G_{k}\\
y_{k+1} &=(1-\alpha_{k+1})y_k + \alpha_{k+1}x_{k+1}-\frac{(1-\alpha_{k+1})\alpha_{k}}{\gamma_{k}+\beta_{k}h_{k}}G_{k}
\end{align*}
\caption{ Adaptive ASGD rewritten \label{alg:a2grad-imp}}
\end{algorithm}

 \paragraph{Parameter estimation.} In order to run \textsc{A2Grad},  two groups of parameters are left for estimation.
\begin{itemize}
  \item Lipschitz constant $L$. This part controls the convergence of deterministic part, since $\gamma_k$ is related to $L$, it also explicitly affects each feature element. Lipschitz constant is in general unknown, but can be estimated in different ways.  On nonconvex model, we choose several initial values empirically and perform grid search to find the best one.
  \item $\beta$ and $h_k$. They effectively impose a decaying stepsize for each feature element. $\beta$ can be chosen by grid search. $h_k$ depends on the empirical variance, a good heuristic solution is to apply moving average of the past stochastic gradient. Let $\tilde{g_k}=\frac{1}{k+1}\sum_{\tau=0}^k\underline{G}_\tau$ and let $\delta_k=\underline{G}_k-\tilde{g}_k$.
\end{itemize}

\paragraph{Logistic regression on \texttt{MNIST}}
In the first experiment, we investigate algorithm performance for the multi-class logistic regression problem. The examined \texttt{MNIST} dataset consists of 
784-d image vectors that are grouped into 10 classes. The size of mini-batch is 128, as suggested in the literature. 
\paragraph{Two-layer neural network on \texttt{MNIST}}
In this experiment, we construct a neural network with two fully-connected (FC) linear layers with 1000 hidden units, which are connected by ReLU activation. We test the model on MNIST dataset with batch size of 128 and and use cross-entropy loss function.
\paragraph{Deep neural networks on \texttt{CIFAR10}}
The goal of this experiment is to classify images into one of the 10 categories. 
 \texttt{CIFAR10} dataset consists of 60000 $32\times 32$ images and uses 10000 of them for testing, and our testing architectures include \textsc{Cifarnet} and \textsc{Vgg16}.  
 \textsc{Cifarnet}  involves 2 convolutional (conv.) layers and FC layers. 
 Conv. layers. are followed by layer response normalization and max pooling layer, and  a dropout layer with keep probability $0.5$ is applied between the FC layers. 
 \textsc{Vgg} (\cite{simonyan2015very}) are deep convolutional neural networks with increasing depth but small kernel-size ($3\times3 $) filters. 
 In particular, we adopt the architecture of \textsc{Vgg16} which consists of 13 conv. layers and 3 FC layers. 
 
 The hyper-parameters are chosen by grid search. For both \textsc{Adam} and \textsc{AMSGrad}, we set $\beta_1=0.9$ and choose $\beta_2$ from grid $\{0.99, 0.999\}$. 
 Learning rate is chosen from the grid $\{0.0001, 0.001, 0.01, 0.1\}$ and momentum parameter is from the grid $\{0.5, 0.9\}$. For our methods, we pick $L$ from the grid $\{0.1, 1, 10\}$ and $\beta_0$ from the grid $\{10, 50, 100, 1000\}$.
We plot the experimental results on the average of 5 runs in Figure \ref{fig:experiment}. On the test of deep neural nets  on \texttt{CIFAR10}, we uses fixed stepsize as  suggested by the recent  paper~\cite{RN308}.
In all the experiments, we find that \textsc{A2Grad} has competitive performance compared with state of the art adaptive gradients \textsc{Adam} and \textsc{AMSGrad}, and often achieve better optimization performance. 
The advantage in optimization performance is also translated into the advantage in generalization performance.

\begin{figure}

\begin{subfigure}{0.33\textwidth}
\includegraphics[width=0.9\linewidth, height=4cm]{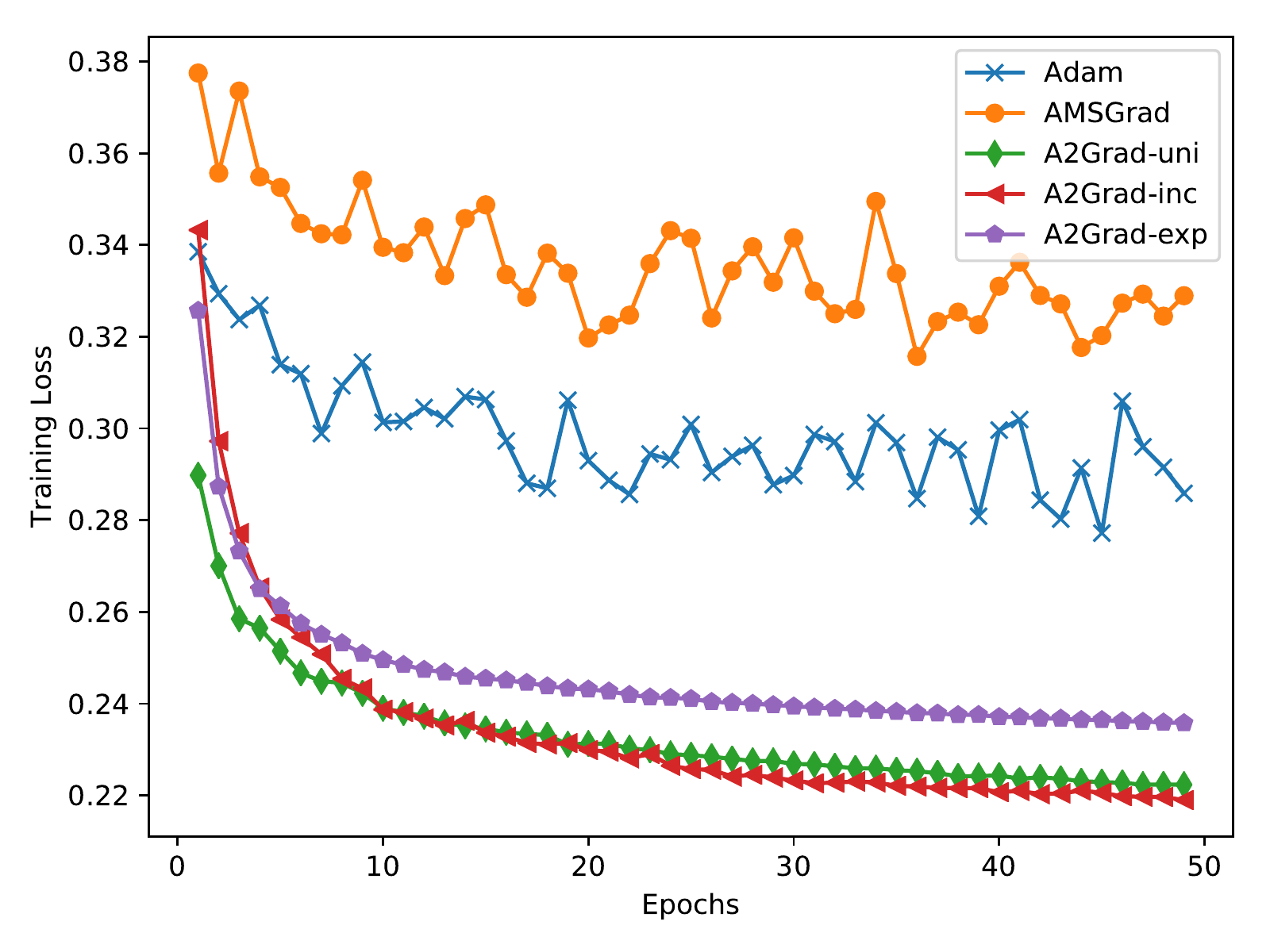}
\caption{}
\label{fig:logreg training loss}
\end{subfigure}
\begin{subfigure}{0.33\textwidth}
\includegraphics[width=0.9\linewidth, height=4cm]{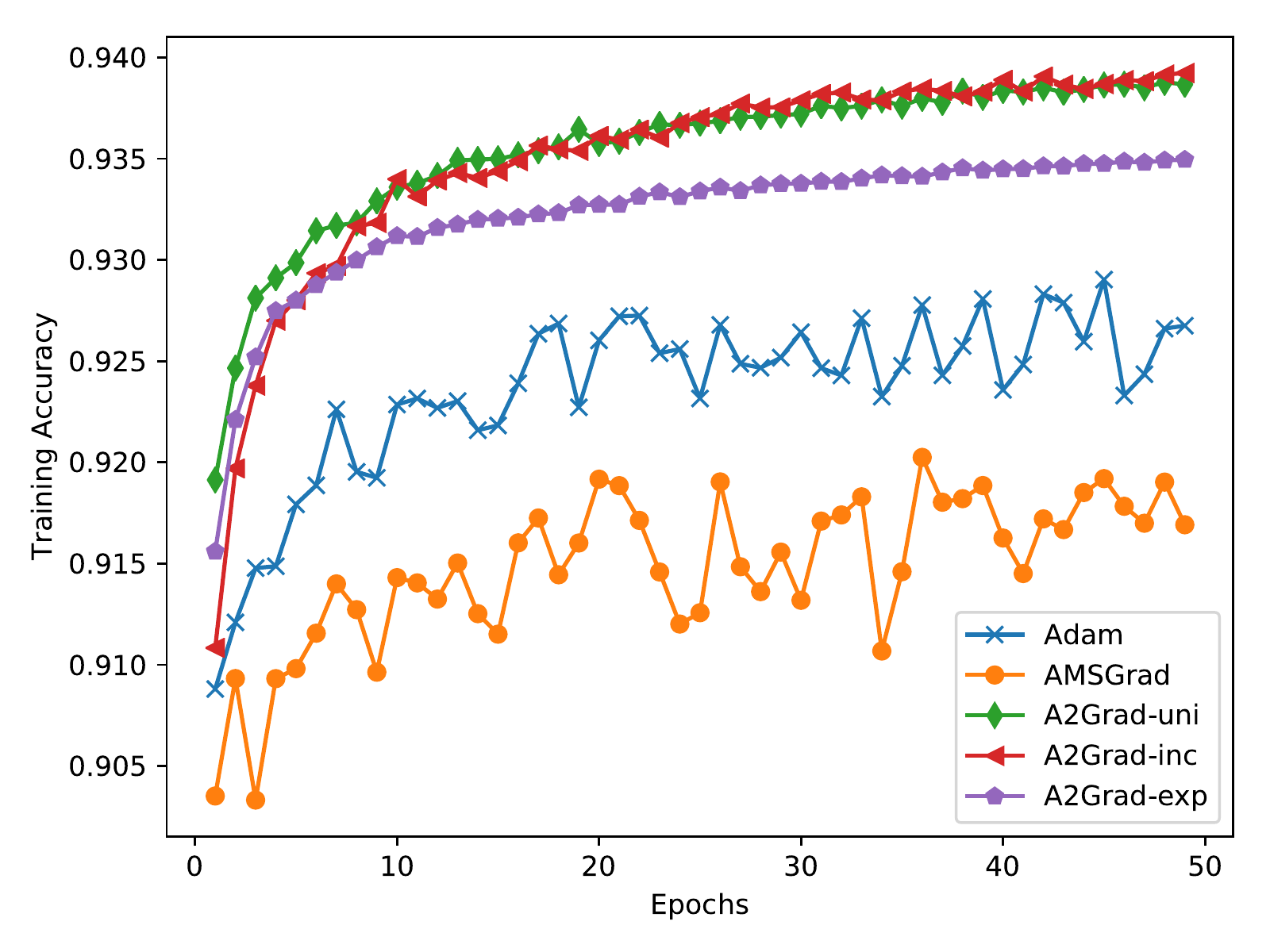}
\caption{}
\label{fig:logreg training acc}
\end{subfigure}
\begin{subfigure}{0.33\textwidth}
\includegraphics[width=0.9\linewidth, height=4cm]{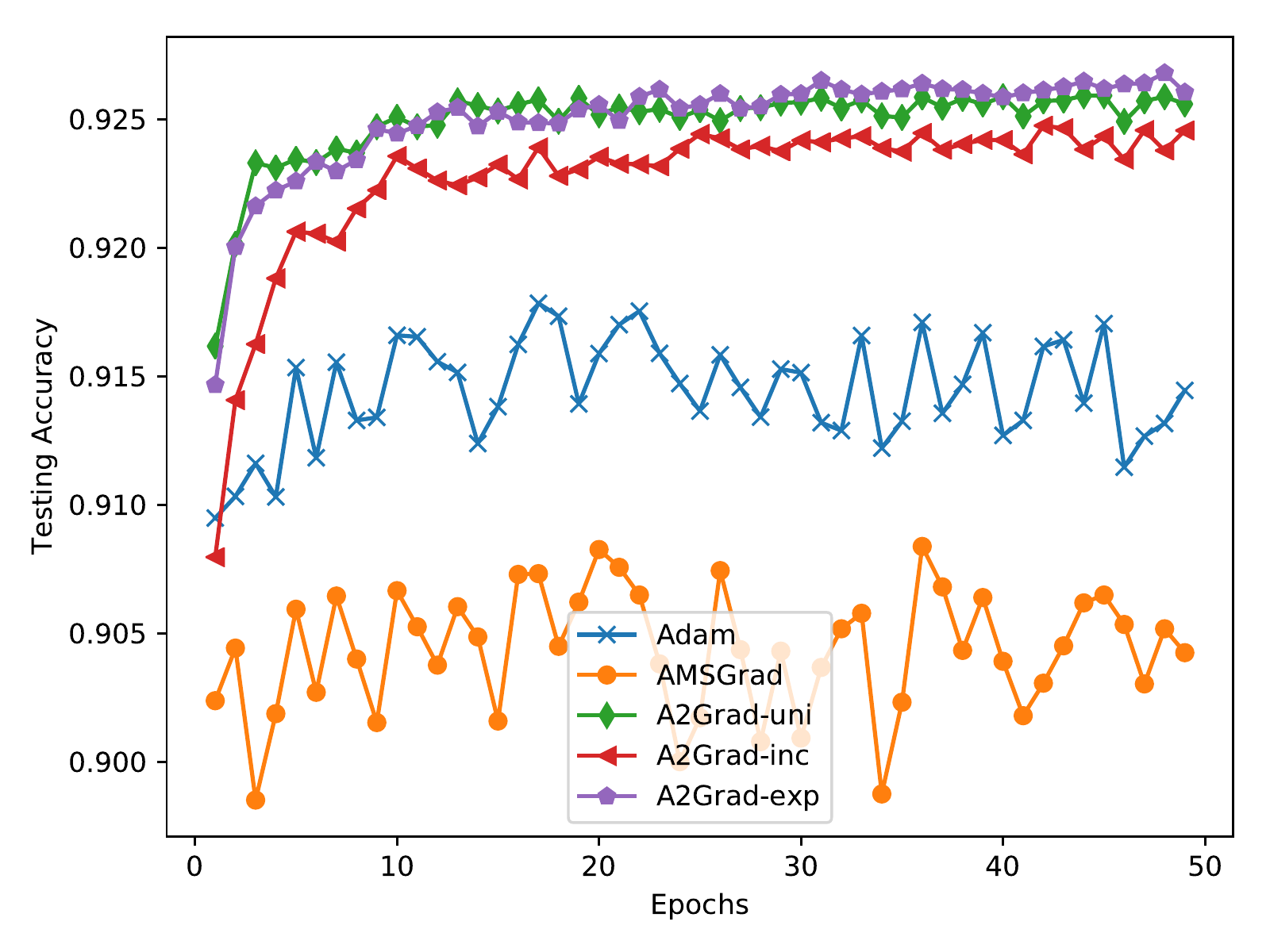}
\caption{}
\label{fig:logreg test acc}
\end{subfigure}

\begin{subfigure}{0.33\textwidth}
\includegraphics[width=0.9\linewidth, height=4cm]{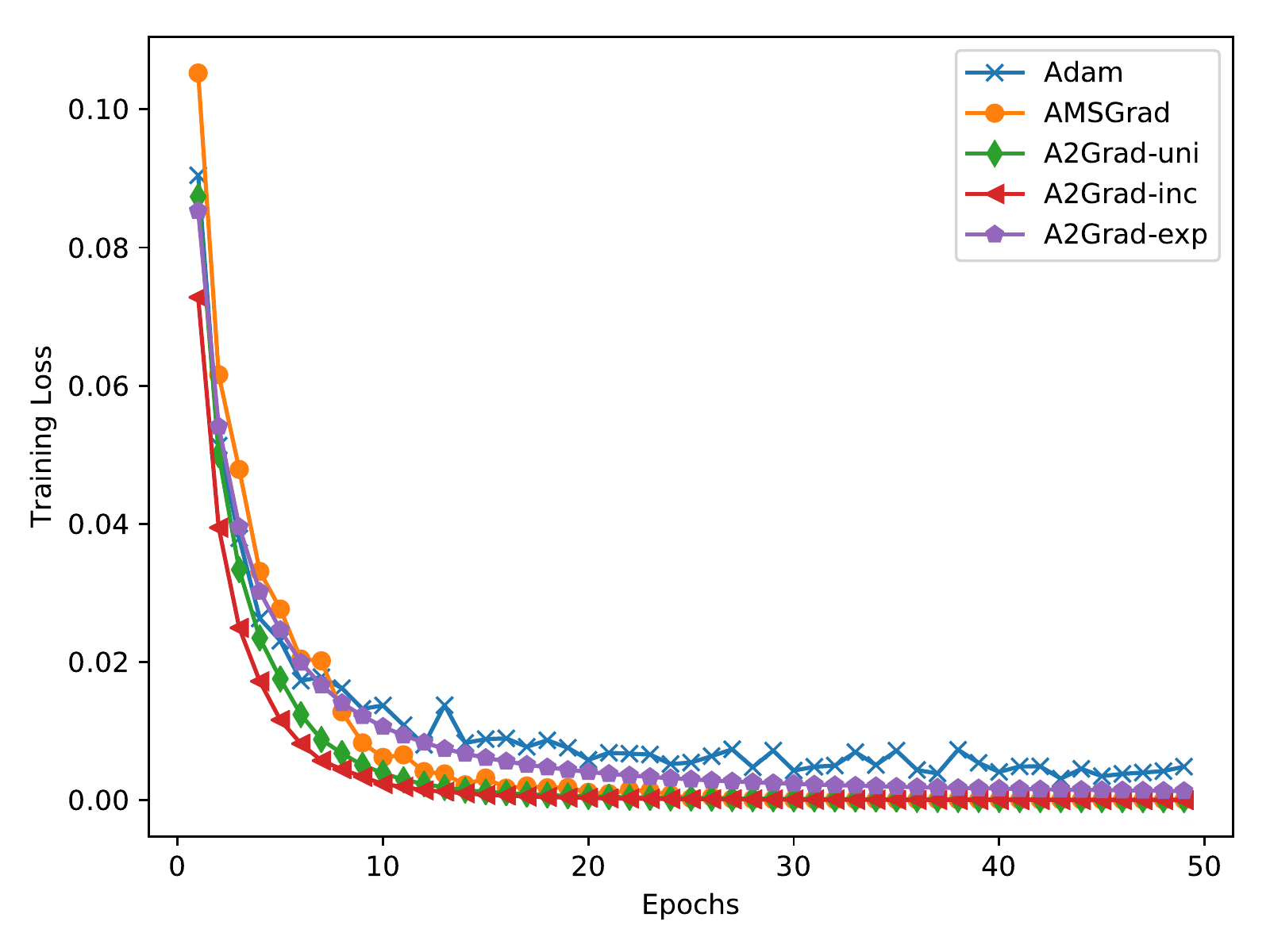}
\caption{}
\label{fig:nn training loss}
\end{subfigure}
\begin{subfigure}{0.33\textwidth}
\includegraphics[width=0.9\linewidth, height=4cm]{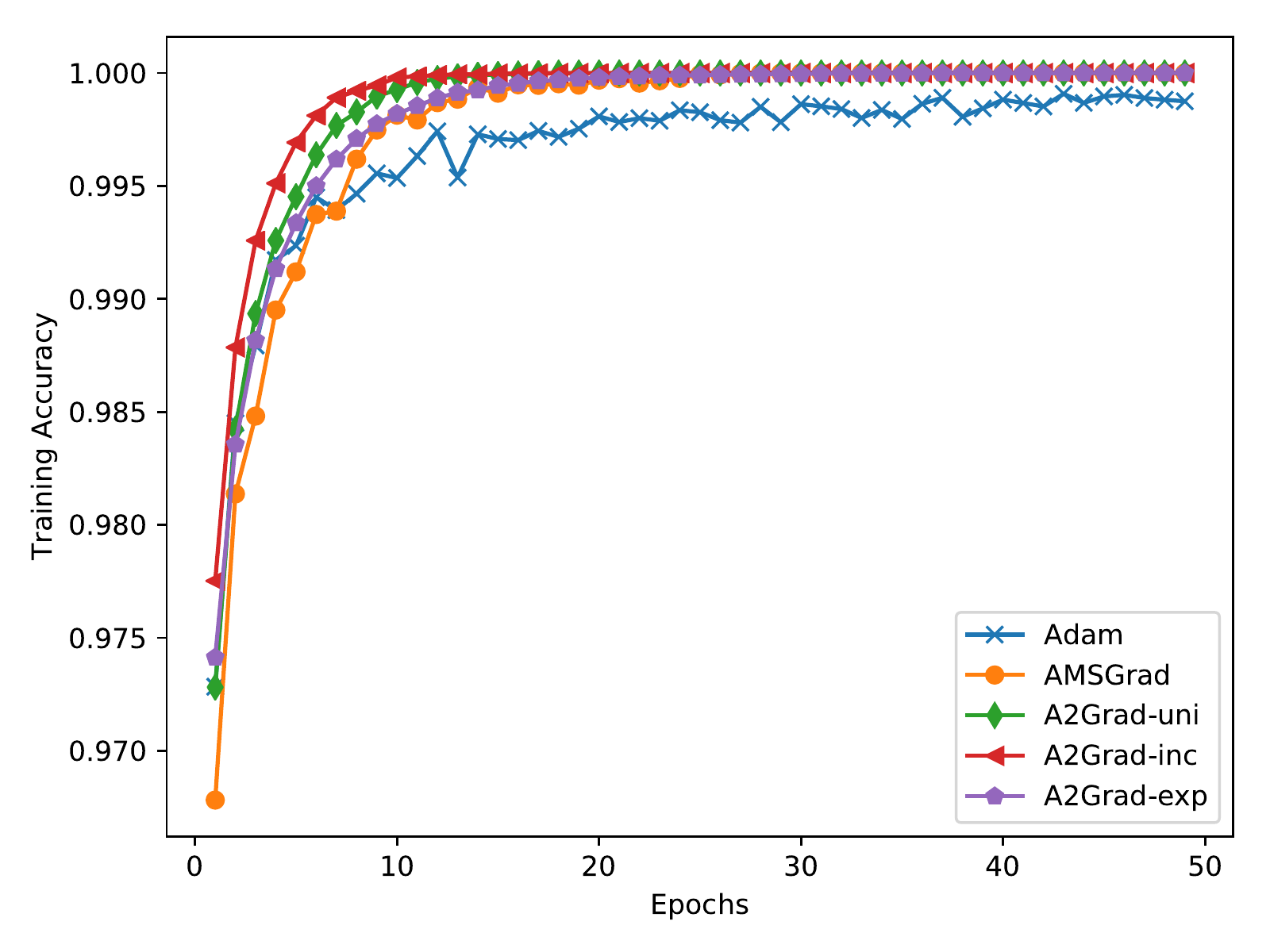}
\caption{}
\label{fig:nn training acc}
\end{subfigure}
\begin{subfigure}{0.33\textwidth}
\includegraphics[width=0.9\linewidth, height=4cm]{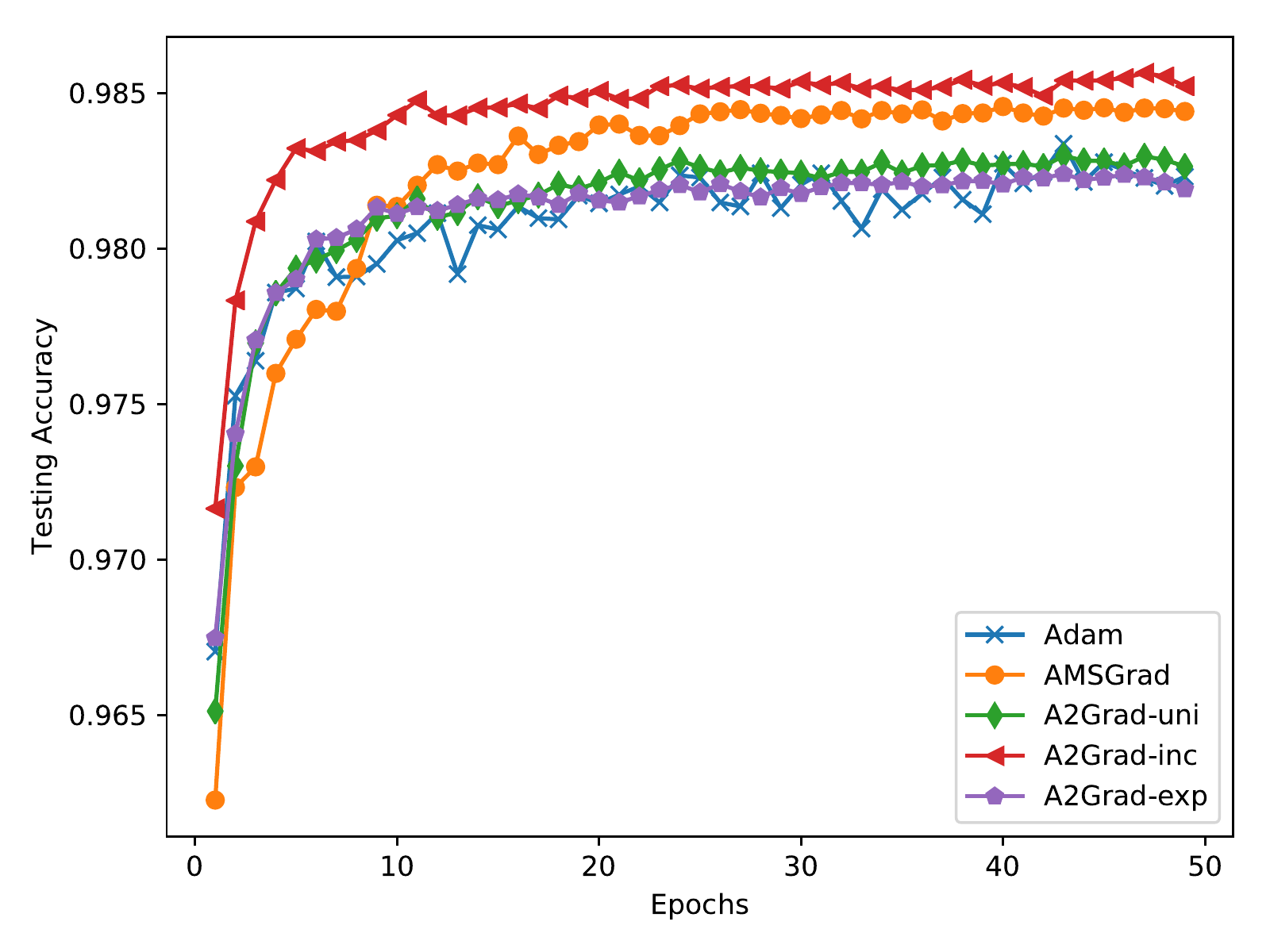}
\caption{}
\label{fig:nn test acc}
\end{subfigure}

\begin{subfigure}{0.33\textwidth}
\includegraphics[width=0.9\linewidth, height=4cm]{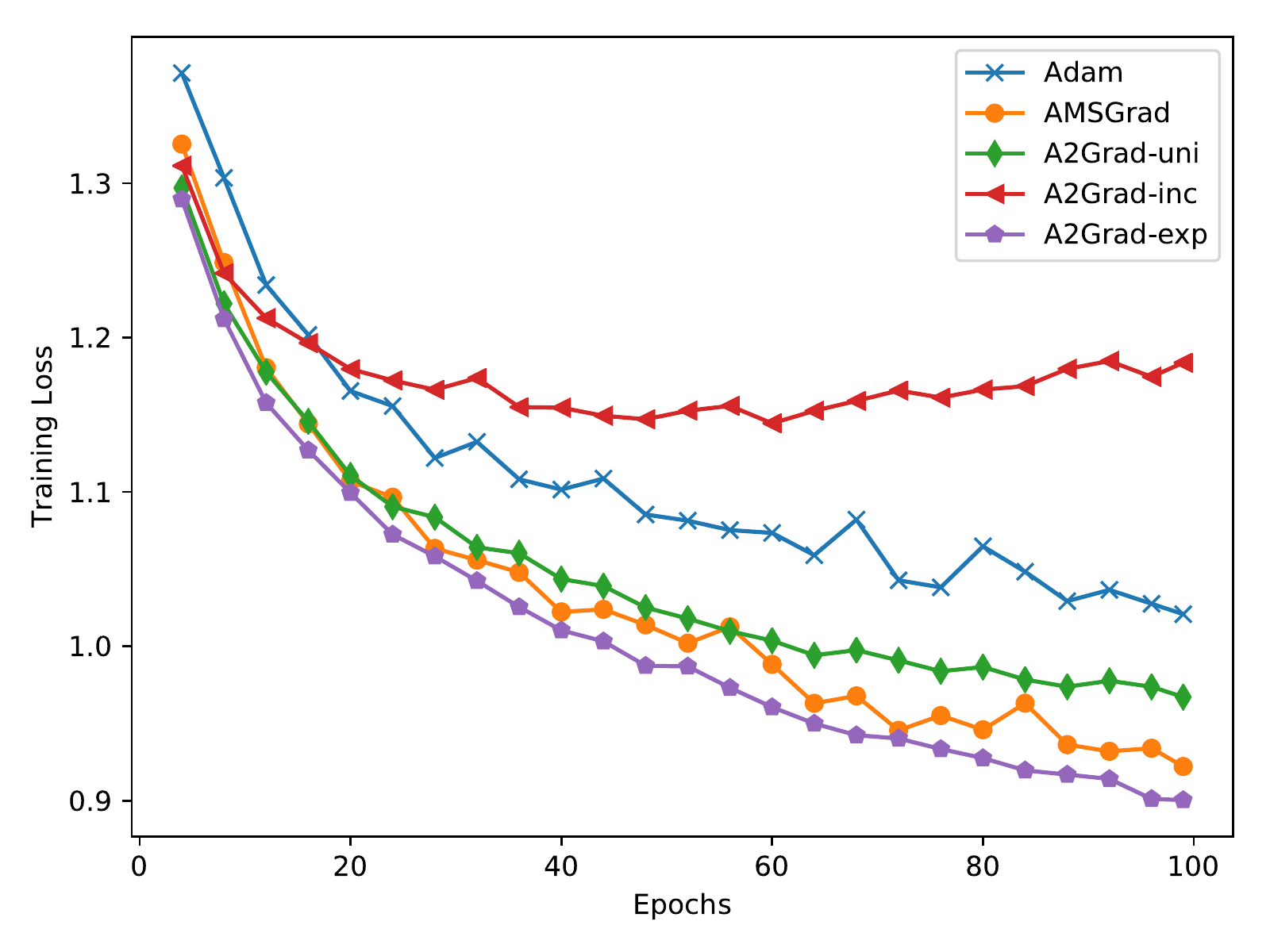}
\caption{}
\label{fig:cifarnet training loss}
\end{subfigure}
\begin{subfigure}{0.33\textwidth}
\includegraphics[width=0.9\linewidth, height=4cm]{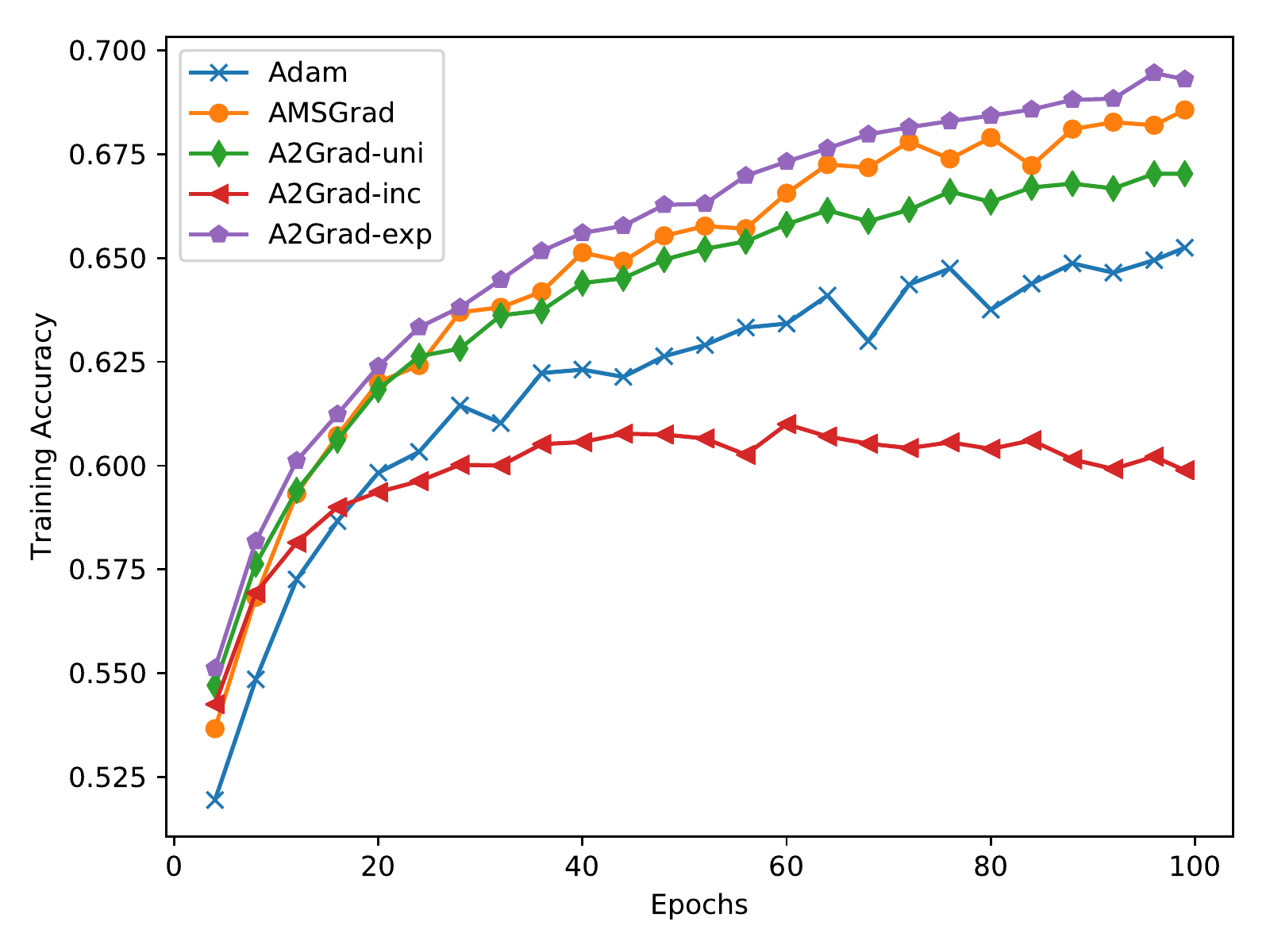}
\caption{}
\label{fig:cifarnet training acc}
\end{subfigure}
\begin{subfigure}{0.33\textwidth}
\includegraphics[width=0.9\linewidth, height=4cm]{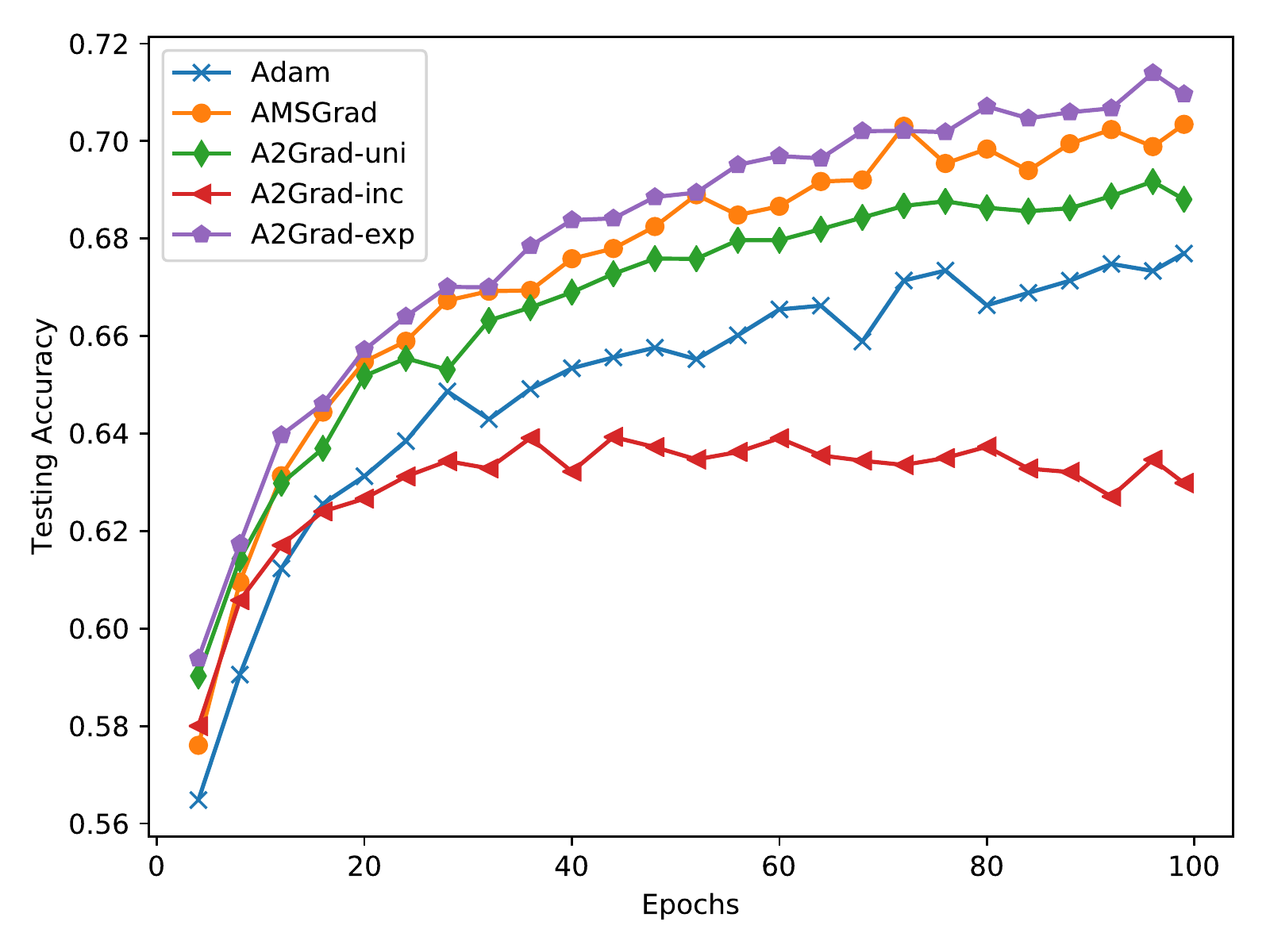}
\caption{}
\label{fig:cifarnet test acc}
\end{subfigure}

\begin{subfigure}{0.33\textwidth}
\includegraphics[width=0.9\linewidth, height=4cm]{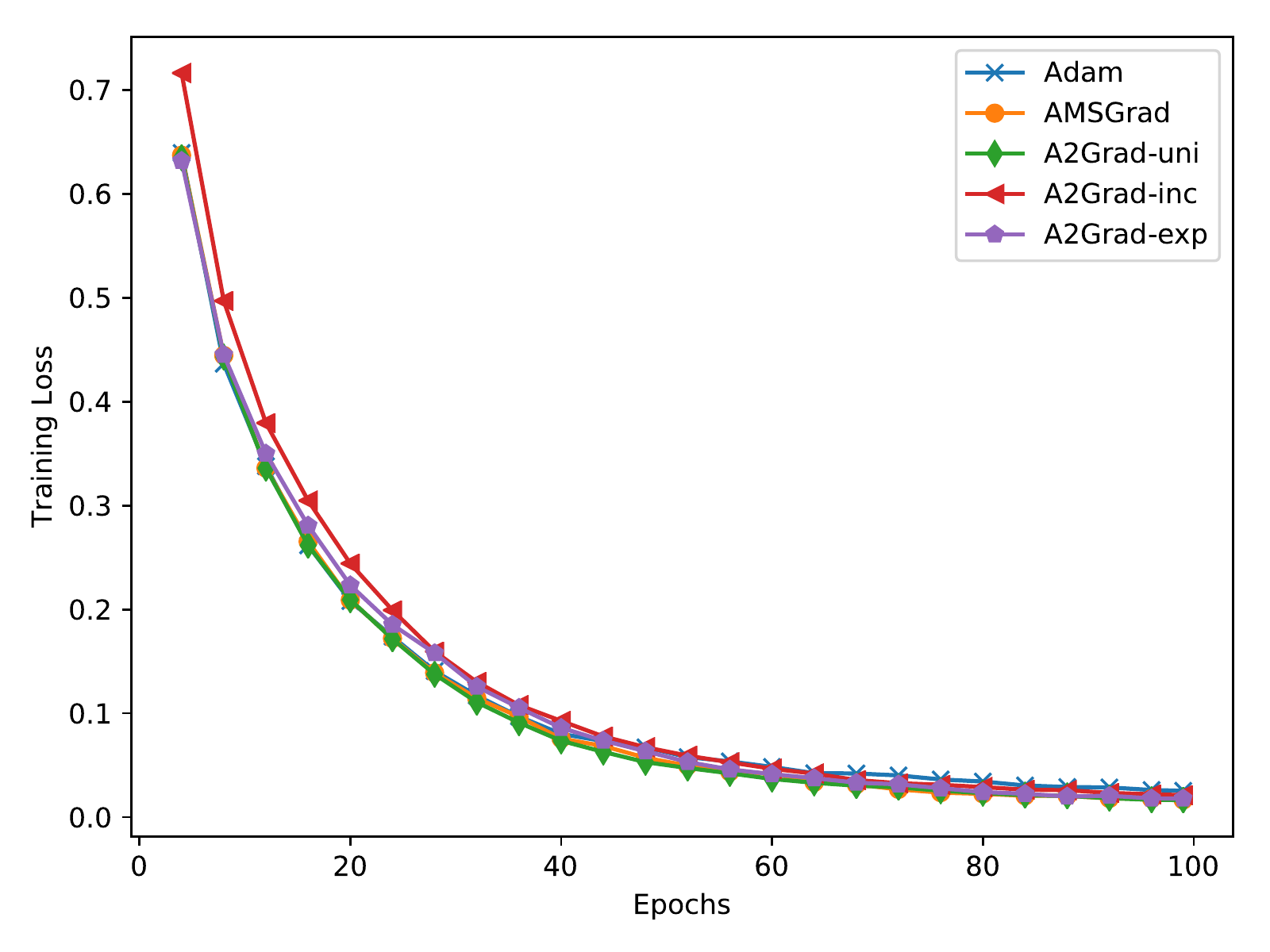}
\caption{}
\label{fig:vgg training loss}
\end{subfigure}
\begin{subfigure}{0.33\textwidth}
\includegraphics[width=0.9\linewidth, height=4cm]{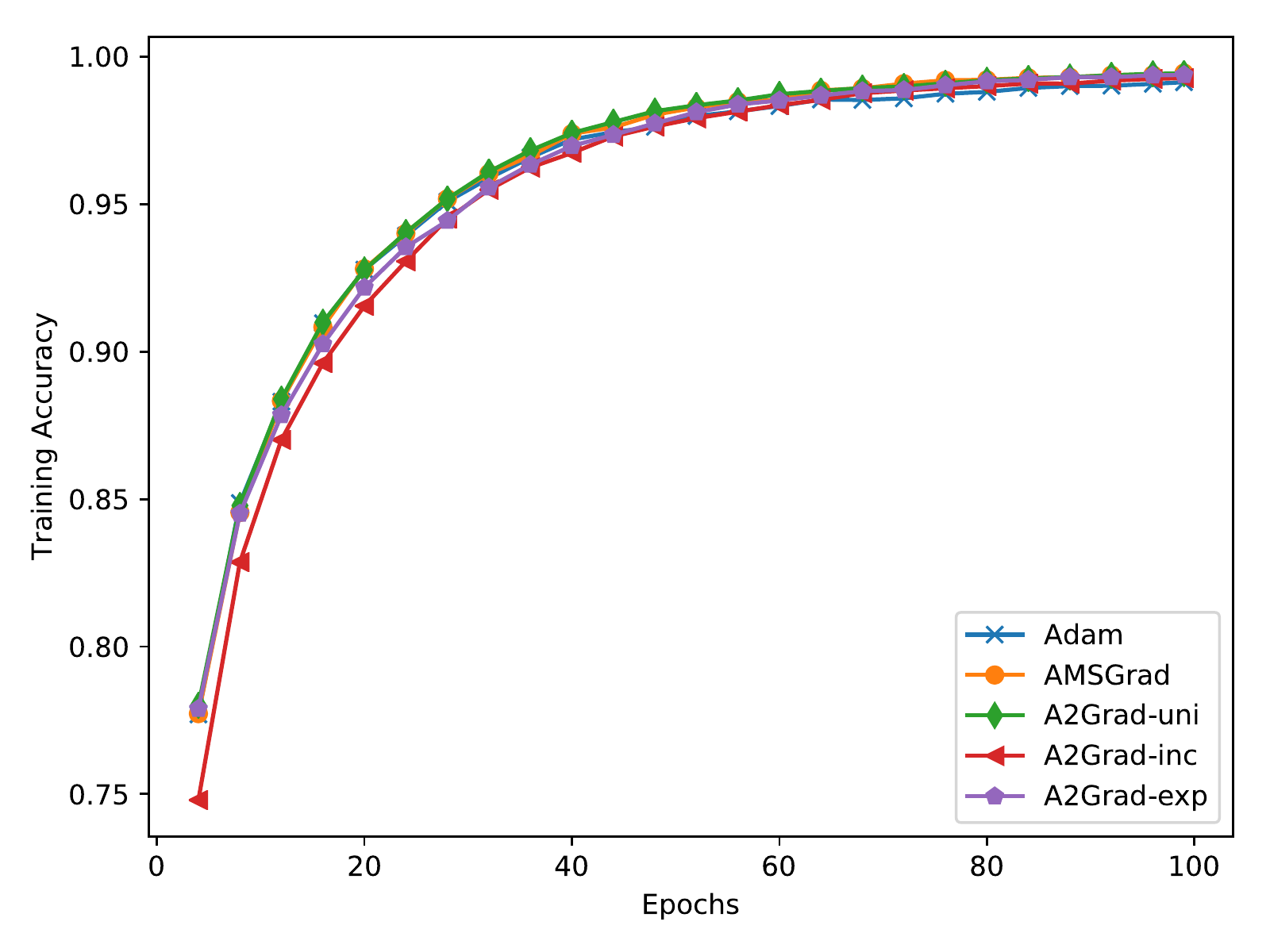}
\caption{}
\label{fig:vgg training acc}
\end{subfigure}
\begin{subfigure}{0.33\textwidth}
\includegraphics[width=0.9\linewidth, height=4cm]{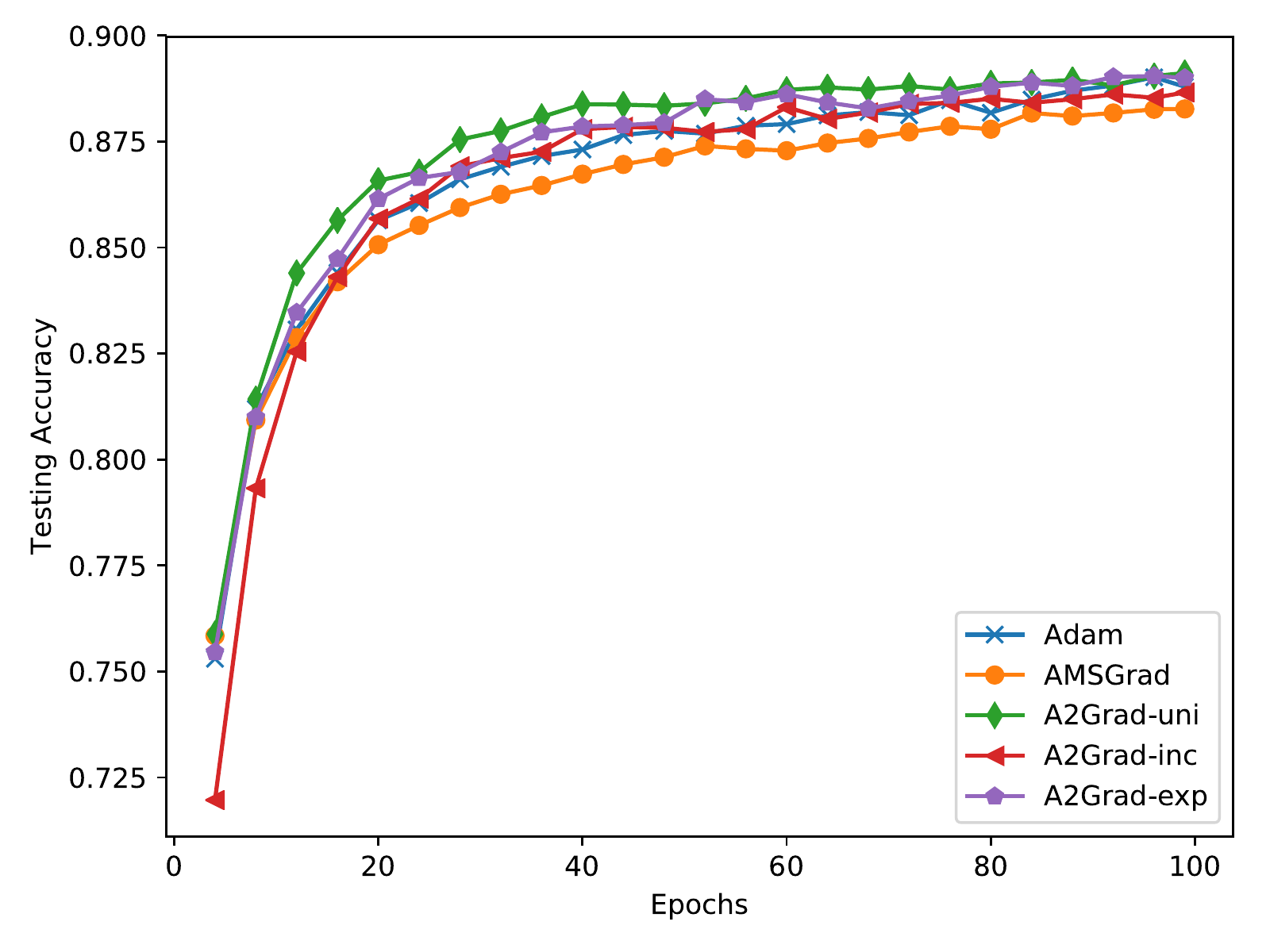}
\caption{}
\label{fig:vgg test acc}
\end{subfigure}
\caption{ Comparison of algorithm performance on convex and nonconvex models. 
From left to right, each column represents the training loss, training accuracy and test accuracy. The first and second row plot the experimental results for logistic regression and neural network on the \texttt{MNIST} dataset respectively. The third and fourth rows plot the results for \textsc{Cifarnet} and \textsc{Vgg16} on \texttt{CIFAR10}  respectively.}
\label{fig:experiment}
\end{figure}

\section{Conclusion\label{sec:Conclusion}}

This paper develops a new framework of accelerated stochastic gradient methods with new adaptive gradients and moving average schemes. 
The primary goal is to tackle the issue of existing adaptive methods and develop more efficient adaptive accelerated methods.
In contrast to the earlier work on adaptive methods, we provide novel analysis to decompose the complexity contributed by adaptive gradients and Nesterov's momentum. 
Our theory gives new insight to the interplay of adaptive gradient and momentum and inspires us to design new adaptive diagonal function.
By choosing adaptive function properly, we further develop new schemes of moving average---incorporating uniform average and exponential average but providing more---with unified theoretical convergence analysis.
Our proposed algorithms not only achieve the optimal worst case complexity rates with respect to both the momentum part and stochastic part, but also demonstrate their empirical advantage over state of the art adaptive methods through experiments on both convex and nonconvex problems. 

\newpage
\medskip

{

\bibliographystyle{plain}
\bibliography{arxiv_sgd.bib}

\newpage
\section*{Appendix}

\subsection*{Proof of Theorems}
Before proving Theorem \ref{thm:main-thm}, we present a version of the three-point Lemma as follows:
\begin{lem}
\label{three-point}
  Let $\phi$ and $\psi$ be two proximal functions and 
  \begin{equation*}
    z = \argmin_{x\in \mathcal{X}} \{\langle g, x\rangle +\alpha D_\phi(y_1, x) +\beta D_\psi(y_2,x)\}.
  \end{equation*}
  Then $\forall x \in \mathcal{X}$, one has
  \begin{dmath*}
    \langle g, z\rangle + \alpha D_\phi(y_1, z) + \beta D_\psi(y_2, z)  \le \langle g, x\rangle   + \alpha D_\phi(y_1, x)
    +\beta D_\psi(y_2, x)
    -  \alpha D_\phi(z, x) - \beta D_\psi(z, x).
  \end{dmath*}
\end{lem}

\main*
\begin{proof}
First let us denote $g_{k}=\nabla f(\underline{x}_{k})$ for simplicity, we have
\begin{align}
  f(\bar{x}_{k+1}) 
  & 
  \le f(\underline{x}_{k})+\left\langle g_k,\bar{x}_{k+1}-\underline{x}_{k}\right\rangle +\frac{L}{2}\|\bar{x}_{k+1}-\underline{x}_{k}\|^{2} \nonumber \\
  &
  = f(\underline{x}_{k})+(1-\alpha_{k})\left\langle g_k,\bar{x}_{k}-\underline{x}_{k}\right\rangle \nonumber\\
  &
  +\alpha_{k}\left\langle g_k,x_{k+1}-\underline{x}_{k}\right\rangle +\frac{L\alpha_{k}^{2}}{2}\|x_{k+1}-x_{k}\|^{2}\nonumber\\
  &
  \le(1-\alpha_{k})f(\bar{x}_{k})+\alpha_{k}[f(\underline{x}_{k})+\left\langle \underline{G}_{k},x_{k+1}-\underline{x}_{k}\right\rangle ]\nonumber\\
  &
  +\frac{L\alpha_{k}^{2}}{2}\|x_{k+1}-x_{k}\|^{2}-\alpha_{k}\left\langle \delta{}_{k},x_{k+1}-\underline{x}_{k}\right\rangle  \label{first-step}
\end{align}
In light of Lemma \ref{three-point}, we obtain the mirror descent step:
\begin{dmath}
\left\langle \underline{G}_{k},x_{k+1}-\underline{x}_{k}\right\rangle 
\le\left\langle \underline{G}_{k},x-\underline{x}_{k}\right\rangle
 +\gamma_{k}\{D(x_{k},x)-D(x_{k+1},x)
 -D(x_{k},x_{k+1})\} +\beta_{k}\{D_{\phi_{k}}(x_{k},x)
 -D_{\phi_{k}}(x_{k+1},x) -D_{\phi_{k}}(x_{k},x_{k+1})\}  
 \label{eq:main-middle-1}
\end{dmath}
Next applying convexity of $f(x)$, we have 
\begin{align}
f(\underline{x}_{k})+\left\langle \underline{G}_{k},x-\underline{x}_{k}\right\rangle  
&
=f(\underline{x}_{k})+\left\langle g_{k},x-\underline{x}_{k}\right\rangle +\left\langle \delta_{k},x-\underline{x}_{k}\right\rangle\nonumber\\
&
\le f(x)+\left\langle \delta_{k},x-\underline{x}_{k}\right\rangle . 
\label{eq:main-middle-2}
\end{align}

Putting (\ref{first-step}), \eqref{eq:main-middle-1} 
and \eqref{eq:main-middle-2} together, we arrive at
\begin{align*}
f(\bar{x}_{k+1}) 
 & 
 \le(1-\alpha_{k})f(\bar{x}_{k})+\alpha_{k}f(x)
 -\frac{1}{2}(\alpha_{k}\gamma_{k}-L\alpha_{k}^{2})\|x_{k+1}-x_{k}\|^{2}\\
 & 
 \quad +\alpha_{k}\gamma_{k}\{D(x_{k},x)-D(x_{k+1},x)\}
 +\alpha_{k}\beta_{k}\left\{ D_{\phi_{k}}(x_{k},x)-D_{\phi_{k}}(x_{k+1},x)\right\} \\
 & 
  \quad +\alpha_{k}\langle\delta_{k},x_{k}-x_{k+1}\rangle-\alpha_{k}\beta_{k}D_{\phi_{k}}(x_{k},x_{k+1})\\
 &
  \quad + \alpha_{k}\left\langle \delta_{k},x-x_{k}\right\rangle \\
 & 
 \le (1-\alpha_{k})f(\bar{x}_{k}) + \alpha_kf(x) 
 -\frac{1}{2}\alpha_k\beta_k\|x_k-x_{k+1}\|^2_{\phi_k}\\
 &
  \quad +\alpha_k\langle\delta_k,x_k-x_{k+1}\rangle + \alpha_k\langle\delta_k,x-x_k\rangle\\
 & 
  \quad+\alpha_{k}\gamma_{k}\{D(x_{k},x)-D(x_{k+1},x)\}
  +\alpha_{k}\beta_{k}\left\{ D_{\phi_{k}}(x_{k},x)-D_{\phi_{k}}(x_{k+1},x)\right\}\\
 &
 \le(1-\alpha_{k})f(\bar{x}_{k}) + \alpha_kf(x)
 -\frac{1}{2}\alpha_k\beta_k\|x_k-x_{k+1}\|^2_{\phi_k}\\
 &
  \quad + \alpha_k\|\delta_k\|_{\phi_{k}^*}\|x_k-x_{k+1}\|_{\phi_k}+\alpha_k\langle\delta_k,x-x_k\rangle\\
 & 
  \quad +\alpha_{k}\gamma_{k}\{D(x_{k},x)-D(x_{k+1},x)\}
  +\alpha_{k}\beta_{k}\left\{ D_{\phi_{k}}(x_{k},x)-D_{\phi_{k}}(x_{k+1},x)\right\}\\
 &
 \le(1-\alpha_{k})f(\bar{x}_{k})+\alpha_{k}f(x)+\alpha_{k}\|\delta_{k}\|_{\phi_{k}*}^{2}/(2\beta_{k})\\
 & 
  \quad +\alpha_{k}\gamma_{k}\{D(x_{k},x)-D(x_{k+1},x)\}
   +\alpha_{k}\beta_{k}\left\{ D_{\phi_{k}}(x_{k},x)-D_{\phi_{k}}(x_{k+1},x)\right\} \\
 & 
  \quad +\alpha_{k}\left\langle \delta_{k},x-x_{k}\right\rangle .
\end{align*}

{\color{red} In the first and second inequalities we use the strong convexity of proximal functions: $D(x,y) \ge \frac{1}{2}\|x-y\|^2$ and  $D_{\phi_k}(x,y) \ge \frac{1}{2}\|x-y\|^2_{\phi_k}$ } and the fact that $-\frac{1}{2}(\alpha_k\gamma_k-L\alpha_k^2)\|x_k-x_{k+1}\|^2 \le 0$ by the relation (\ref{eq:alpha_gamma_bound}); in the third inequality, we apply Cauchy–Schwarz inequality as $\langle\delta_k,x_k-x_{k+1}\rangle \le \|\delta_k\|_{\phi_{k}^*}\cdot\|x_k-x_{k+1}\|_{\phi_k}$; in the last inequality, we use the fact that for $a>0$,  $bx-\frac{a}{2}x^2\le \frac{b^2}{2a}$.

Finally, summing up the above relation for $k=0,1,2,...$ with each side weighted by $\lambda_k$, and using relation (\ref{eq:alpha_gamma_mono}), we have
\begin{dmath}
\lambda_{K}\left[f(\bar{x}_{K+1})-f(x)\right]  
\le(1-\alpha_{0})\left[f(\bar{x}_{0})-f(x)\right] 
+\alpha_{0}\gamma_{0}D(x_{0},x) 
 +\sum_{k=0}^{K}\lambda_{k}[\frac{\alpha_{k}\|\delta_{k}\|_{\phi_{k}*}^{2}}{2\beta_{k}} 
 +\alpha_{k}\left\langle \delta_{k},x-x_{k}\right\rangle +\alpha_{k}R_{k}],  
\label{main-recur2}
\end{dmath}
where $R_{k}=\beta_{k}D_{\phi_k}(x_{k},x)-\beta_{k}D_{\phi_k}(x_{k+1},x).$
\end{proof}
\normalsize

\diagcvg*
\begin{proof} 
One has
\begin{align*}
 & \sum_{k=0}^{K}\lambda_{k}\alpha_{k}\left\{ \beta_{k}D_{\phi_k}(x_{k},x)-\beta_{k}D_{\phi_k}(x_{k+1},x)\right\} \\
{\le} & \sum_{k=1}^{K}\left[\lambda_{k}\alpha_{k}\beta_{k}D_{\phi_k}(x_{k},x)-\lambda_{k-1}\alpha_{k-1}\beta_{k-1}D_{\phi_{k-1}}(x_{k},x)\right]\\
&
+\lambda_{0}\alpha_{0}\beta_{0}D_{\phi_{0}}(x_{0},x)\\
\overset{\text{(\ref{monotone})}}{\le}
& \sum_{k=1}^{K}\sum_{i=1}^{d}\left[\lambda_{k}\alpha_{k}\beta_{k}h_{k, i}-\lambda_{k-1}\alpha_{k-1}\beta_{k-1}h_{k-1,i}\right]B\\
&
+\lambda_{0}\alpha_{0}\beta_{0}\sum_{i=1}^{d}h_{0, i}B\\
\le & \lambda_{K}\alpha_{K}\beta_{K}\sum_{i=1}^{d}h_{K, i}B .
\end{align*}
{\color{red}Moreover,  we have $\lambda_{k}=\frac{(k+1)(k+2)}{2}$ by the definition of $\alpha_k$,    and $\mathbb{E} [\langle \delta_{k},x_{k}-x\rangle ] =0$ by taking expectation w.r.t stochastic gradient $\underbar{G}_k$. }
It then remains to plug in the values of $\alpha_k$ and $\lambda_k$ to (\ref{main-recur2}) to prove our result.
\end{proof}
\normalsize

\subsection*{Proof of Corollaries}
\polyavg*
\normalsize
The proof of Corollary \ref{cor:poly-avg}  is built on Lemma 10 in \cite{RN78}. 
For the sake of convenience, we present this lemma as follows:
\begin{lem}
\label{lem:sum_frac_bound}Let $\{a_{\tau}\}_{0\le\tau\le t}$ be
a sequence of real values, and $a_{:\tau}=[a_{0},a_{1},...,a_{\tau}]^{T}$
, then 
\[
\sum_{\tau=0}^{t}\frac{a_{\tau}^{2}}{\|a_{:\tau}\|_{2}}\le2\|a_{:t}\|_{2}.
\]
\end{lem}

\begin{proof}[Proof of Corollary \ref{cor:poly-avg}]
Define $\bar{\delta}_k=(k+1)^{q/2}\delta_k$. Using Lemma \ref{lem:sum_frac_bound}, we arrive at
\begin{align}
\sum_{k=0}^{K}\frac{(k+1)\|\delta_{k}\|_{h_{k}*}^{2}}{\beta_{k}} 
& =\sum_{k=0}^{K}\sum_{i=1}^{d}\frac{(k+1)\delta_{k,i}^{2}}{\beta h_{k,i}}\nonumber\\
 & =\sum_{k=0}^{K}\sum_{i=1}^{d}\frac{(k+1)^{1-q/2}\bar{\delta}_{k,i}^{2}}{\beta \|\bar{\delta}_{:k,i}\|_2}\nonumber\\
 & \le\frac{2(K+1)^{1-q/2}}{\beta}\sum_{i=1}^{d}\|\bar{\delta}_{:k,i}\|. \label{cor-middle-01}
\end{align}
Moreover, by the definition of $\bar{\delta}$, we have
\begin{equation}
  \|\bar{\delta}_{:k,i}\|_2=\sqrt{\sum_{\tau=0}^k (\tau+1)^{q}\delta_{\tau,i}^2}\le (k+1)^{q/2}\|\delta_{:k,i}\|_2,\ 1\le i\le d. \label{cor-middle-02}
\end{equation}
This also gives rise to the bound $h_{k,i}\le \|\delta_{0:k,i}\|_2$.
Putting these together, we have
\begin{dmath*}
\mathbb{E} \left[f(\bar{x}_{K+1})-f(x)\right] 
 \le\frac{2L\|x-x_{0}\|^{2}}{(K+1)(K+2)}
 +\frac{2\beta B\ \mathbb{E}\left[\sum_{i=1}^{d}\|\delta_{0:K,i}\|\right]}{K+2}
 +\frac{2\mathbb{E}\left[\sum_{i=1}^{d}\|\delta_{0:K,i}\|\right]}{\beta(K+2)}
\end{dmath*}
\end{proof}
\normalsize

\expavg*
\begin{proof} Using the definition of $h_k, v_k$, we have 
$$h_k\ge (1-\rho)\sqrt{k+1}\delta_k,\quad k=0,1,2,...$$
and 
$$h_{k,i}\le \sqrt{k+1}\max_{0\le\tau\le k}\left|\delta_{\tau,i}\right|,\quad i=1,2,...,d,\ k=0,1,2,...$$
\begin{align}
\sum_{k=0}^{K}\frac{\lambda_{k}\alpha_{k}\|\delta_{k}\|_{\phi_{k}*}^{2}}{2\beta_{k}} 
& \le \sum_{k=0}^{K}\sum_{i=1}^{d}\frac{\sqrt{k+1}\left|\delta_{k,i}\right|}{2\beta (1-\rho)}\nonumber\\
 & \le\frac{\sqrt{K+1}}{2\beta (1-\rho)}\sum_{i=1}^{d}\|\delta_{:k,i}\|_1. \label{cor-middle-011}
\end{align}
  
In conclusion, we have 
\begin{dmath*}
\mathbb{E}\left[f(\bar{x}_{K+1})-f(x)\right]  
\le
\frac{2L\|x-x_{0}\|^{2}}{(K+1)(K+2)}
+\frac{2\beta B\sum_{i=1}^{d}\mathbb{E}\left[\max_{0\le k\le K}\left|\delta_{k,i}\right|\right]}{\sqrt{K+2}}
+\frac{\sum_{i=1}^{d}\mathbb{E}\left[\|\delta_{0:K,i}\|_1\right]}{2\beta(1-\rho)\sqrt{K+1}(K+2)}.
\end{dmath*}

For the second part, it suffices to show the expectation of some max-type random variables. 
For brevity, we use $\delta_{k, i}$  $\delta_{i}$ exchangeably.
Under the sub-Gaussian assumption, the moment generating function of $\delta_{i}$---$M_t$ satisfies
$$M_t=\mathbb{E}e^{t |\delta_i|}\le 2 e^{t^2 \bar{\sigma}^2/2},$$
for any $t>0$.  Using some standard analysis, we have
\begin{align*}
\exp\mathbb{E}\left[t\max_{0\le k\le K}\left|\delta_{k,i}\right|\right] 
& \le\mathbb{E}\left[\exp\left(t\max_{0\le k\le K}\left|\delta_{k,i}\right|\right)\right] \\
& =\mathbb{E}\left[\max_{0\le k\le K}\exp \left(t\left|\delta_{k,i}\right|\right)\right]\\
& \le\sum_{0\le k\le K}\mathbb{E}\left[\exp t\left|\delta_{k,i}\right|\right] \\
& =(K+1)M_{t}
\end{align*} 
where the first inequality is from Jensen's inequality.
Taking logarithm on both sides yields 
$$\mathbb{E}\left[\max_{0\le k\le K}\left|\delta_{k,i}\right|\right]\le\frac{\log(2(K+1))+t^2\bar{\sigma}^2/2}{t},\quad \forall t. $$ 
Choosing $t=\frac{\sqrt{2\log(2(K+1))}}{\bar{\sigma}}$, 
we have $$\mathbb{E}\left[\max_{0\le k\le K}\left|\delta_{k,i}\right|\right]
\le \sqrt{2\log(2(K+1))} \bar{\sigma}.
$$
Moreover, by sub-Gaussian property, we have $\mathbb{E}|\delta_i|\le\bar{\sigma}\sqrt{2\pi}$.
Overall, we have 
\begin{dmath*}
\mathbb{E}\left[f(\bar{x}_{K+1})-f(x)\right] 
= \frac{2L\|x-x_0\|^2}{(K+1)(K+2)}
+ 2\beta B\frac{\sqrt{2\log(2(K+1))}\bar{\sigma}}{\sqrt{K+2}}
+\frac{\sqrt{2\pi}d \bar{\sigma}}{2\beta(1-\rho)\sqrt{K+2}} .
\end{dmath*}
For the last part, by assuming $\|G(x, \xi)\|_\infty \le C$ we will have $\|\delta_k\|_\infty\le 2C$, $0\le k \le K$. Immediately, we draw the conclusion by plugging into (\ref{eq:converge-exp-avg}).
\end{proof}
\normalsize

\end{document}